\documentclass{article}

\usepackage{arxiv}
 \usepackage[numbers]{natbib} 
\setcitestyle{square}
\usepackage{mathtools} 
\usepackage{booktabs} 
\usepackage{amsthm}
\usepackage{amsmath}
\usepackage{amssymb}

\usepackage{shortcuts}
\usepackage{tikz}
\usepackage{thmtools}
\usepackage{thm-restate}
\usepackage{subfig,caption,graphicx}
\usepackage[ruled,vlined]{algorithm2e}
\usepackage{listings}

\newcount\Comments  
\definecolor{darkgreen}{rgb}{0,0.5,0}
\definecolor{darkred}{rgb}{0.7,0,0}
\definecolor{teal}{rgb}{0.3,0.8,0.8}
\definecolor{blue}{rgb}{0,0,1}
\definecolor{purple}{rgb}{0.5,0,1}
\newcommand{\kibitz}[2]{\ifnum\Comments=1\textcolor{#1}{#2}\fi}

\newcommand{\ci}{\perp\!\!\!\perp}
\theoremstyle{plain}
\newtheorem{theorem}{Theorem}

\newtheorem{corollary}[theorem]{Corollary}
\theoremstyle{definition}

\newtheorem{definition}{Definition}

\usetikzlibrary{shapes,decorations,arrows,calc,arrows.meta,fit,positioning}
\tikzset{
    -Latex,auto,node distance =1 cm and 1 cm,semithick,
    state/.style ={ellipse, draw, minimum width = 0.7 cm},
    point/.style = {circle, draw, inner sep=0.04cm,fill,node contents={}},
    square/.style = {rectangle, draw, inner sep=0.04cm,fill,node contents={}},
    bidirected/.style={Latex-Latex,dashed},
    el/.style = {inner sep=2pt, align=left, sloped}
}

\title{A Unifying Causal Framework for Analyzing Dataset Shift-stable Learning Algorithms}

\author{
 Adarsh Subbaswamy \\
  Department of Computer Science\\
  Johns Hopkins University\\
  \texttt{asubbaswamy@jhu.edu} \\
   \And
 Bryant Chen \\
  Brex Inc.\\
  \And
  Suchi Saria\\
  Department of Computer Science\\
  Johns Hopkins University \& Bayesian Health\\
}

\begin{document}
\date{Published May 19, 2022 in the Journal of Causal inference}
\maketitle
\begin{abstract}
Recent interest in the external validity of prediction models (i.e., the problem of different train and test distributions, known as \textit{dataset shift}) has produced many methods for finding predictive distributions that are invariant to dataset shifts and can be used for prediction in new, unseen environments. However, these methods consider different types of shifts and have been developed under disparate frameworks, making it difficult to theoretically analyze how solutions differ with respect to stability and accuracy. Taking a causal graphical view, we use a flexible graphical representation to express various types of dataset shifts. Given a known graph of the data generating process, we show that all invariant distributions correspond to a causal hierarchy of graphical operators which disable the edges in the graph that are responsible for the shifts. The hierarchy provides a common theoretical underpinning for understanding when and how stability to shifts can be achieved, and in what ways stable distributions can differ. We use it to establish conditions for minimax optimal performance across environments, and derive new algorithms that find optimal stable distributions. Using this new perspective, we empirically demonstrate that that there is a tradeoff between minimax and average performance.
\end{abstract}

\section{Introduction}
Statistical and machine learning (ML) predictive models are being deployed in a number of high impact applications, including healthcare \citep{strickland2018hospitals}, law enforcement \citep{winston2018palantir}, and criminal justice \citep{angwin2016machine}. These safety-critical applications have a high cost of failure---model errors can lead to incorrect decisions that have a profound impact on the quality of human lives---which makes it important to ensure that systems being developed and deployed for these problems behave \emph{reliably} (i.e., they perform to their specification). To do so, developers are forced to reason in advance about likely sources of failure and address them prior to deployment (i.e., during model training). A key source of failure is due to \emph{dataset shifts} \citep{candela2009dataset,finlayson2021clinician}: differences between the environment in which training data was collected and the environment in which the model will be deployed that manifest as changes in the data distribution. These differences can arise due to deploying a model at a new site from which data was unavailable during training, or due to natural variations that occur over time. Failing to account for these differences can result in model predictions with worse performance (i.e., expected loss) than anticipated.

Across a number of application domains, the recent COVID-19 pandemic has demonstrated ways in which dataset shifts can induce model failures. For example, the pandemic resulted in a drastic shift in online retail and the consumer packed goods industries: during the onset of the pandemic, the predictive algorithms powering Amazon's supply chain failed due to the sudden increased demand for household supplies (e.g., bottled water and paper products), resulting in unprecedented item shortages and delivery delays \citep{dickson2020coronavirus}.

Beyond changes to customer behavior, dataset shift has been identified as a key challenge to ensuring reliability in safety-critical domains such as healthcare (see, e.g., the example ways in which dataset shift can occur in medical applications in \cite{finlayson2021clinician}). Consider the following examples: Long term (e.g., 3 year) patient mortality prediction models are used to help determine which patients may need long term support after being discharged from the hospital. In one study, the authors trained a model to predict 3 year patient mortality from electronic health record (EHR) data at a single hospital. The authors found that, for 68\% of laboratory tests, the timing of the laboratory test orders was more predictive of mortality for the model than the corresponding values of those tests \cite{agniel2018biases}. As a result, the model learned predictive dependencies between the time of day when a lab test was ordered and patient mortality. These dependencies are brittle: they are highly variant across hospitals because the timing of lab tests is determined by hospital-specific policies and physician-specific preferences \cite{grytten2003practice,cutler2019physician}. Models which have learned these brittle dependencies can experience significant deterioration in performance and become unsafe to use \cite{schulam2017reliable}.

As another example, consider \citep{zech2018variable}, in which the authors trained a model to diagnose pneumonia from chest X-rays. While the model was found to be very accurate on new patients at the medical center where it was developed, this performance deteriorated significantly when applied at new, but similar, medical centers. Their analysis showed that the model had learned dependencies between stylistic features (e.g., text, orientation, coloring) present in the X-ray and pneumonia. These associations varied widely across hospitals because the choice of stylistic features depended on the X-ray equipment, hospital policies, and technician preferences. 

These examples demonstrate that dataset shifts can arise from a variety of changes (i.e., interventions) in the underlying data generating process (DGP) such as changes in behavior (e.g., shifts in clinician treatment patterns) or changes in data acquisition (e.g., new X-ray machines and settings). Preventing failures due to these  kinds of shifts requires a \textit{causal} understanding of the parts of the data generating process that can shift, and learning models of \emph{stable} distributions that are invariant to these shifts.\footnote{We will refer to distributions that are stable to dataset shifts as ``shift-stable'' or simply ``stable''.}.

Given that dataset shifts can happen in nearly every domain where predictive models are used, and given how serious the consequence of failures due to these shifts can be, it is critical to be able to ensure the reliability of models under such shifts: That is, we need to understand under a given set of dataset shifts, how will a model's behavior change? What can be said about a model's \emph{stability} (i.e., are the model's predictions still accurate after the shifts)? In the X-ray example, a model developer wants to know: what shifts can lead to model instability? Would shifts in color encoding schemes or upgrades to X-ray equipment lead to instability and deteriorate model accuracy? Has the model learned any dependencies (e.g., between pneumonia and choice of equipment) that will lead to this instability? For models trained using different algorithms, what guarantees do they give about stability to shifts in color schemes? How do the accuracies of these models differ under these dataset shifts? What guarantees can be made about a model's worst-case performance under such shifts? Lacking common footing for framing stability and dataset shift, it is difficult to begin to answer these questions and compare algorithms.

A common framing of dataset shift is to assume limited data from a clearly defined ``target'' environment or distribution of interest is used (along with more plentiful data from a ``source'' environment) to make inferences about the target environment. This framing allows an analyst to \emph{``reactively''}\footnote{The term ``reactive'' to describe approaches which use (possibly unlabeled) data from the target domain during learning was coined in \citep{subbaswamycounterfactual}.} adjust to target data samples. Reactive approaches to addressing dataset shift exist across multiple fields of study. Some examples include methods for domain adaptation in machine learning (for an overview, see \citep{candela2009dataset}), generalizability and transportability in causal inference (e.g., \citep{pearl2011transportability,stuart2015assessing,bareinboim2016causal,degtiar2021review}), and sample selection bias in statistics and econometrics (e.g., \citep{heckman1974shadow,heckman1979sample,winship1992models,vella1998estimating}). Reactive approaches to model training require data from the target environment which makes it difficult to learn a model from source data alone which  will perform well in a new, unseen environment. In this case, it is important to instead use \emph{proactive} learning approaches which learn models that are stable to any anticipated problematic shifts.

One common class of proactive learning methods is \emph{declarative} in nature. These methods allow users to specify dependencies (i.e., causal relationships) between variables in the dataset that are likely to experience a dataset shift. That is, the data generating process is expected to differ between source and target environments due to an unknown intervention on the specified causal relationships. For instance, in the previously discussed X-ray example, a user might want to specify that the model should be stable to changes in scanner manufacturer, the choice of X-ray orientation (front-to-back vs back-to-front), or the color encoding scheme, since all of these are problematic shifts that are likely to occur. Example learning methods include approaches which find stable subsets of features \citep{magliacane2018domain, subbaswamycounterfactual}, approaches which learn models under hypothetically stable data generating processes \citep{subbaswamy2019transport,subbaswamy2020spec}, and ``counterfactual'' approaches which compute counterfactual features \citep{subbaswamycounterfactual,schulam2017reliable,veitch2021counterfactual} or perform data augmentation \citep{ilse2021selecting} to remove unstable dependencies. A key feature of declarative learning methods is that they can give guarantees about the stability of model predictions to changes in the specified shifts. When a user specifies that they desire invariance to the choice of X-ray orientation, then the declarative method will find a stable solution satisfying this specification. However, this requires domain expertise to be able to specify the likely problematic shifts to which stability is desired. A notable exception is the method of  \citep{subbaswamy2020spec}, which learns candidate shifts that occurred across datasets and allows users to choose which invariances to enforce. Additionally, while different declarative methods can guarantee stability, it is unknown what tradeoffs exist between methods with respect to accuracy. For example, models trained using stable feature subsets vs counterfactuals can both satisfy stability to, e.g., the choice of X-ray orientation, but we currently struggle to answer how their accuracy will compare and differ under shifts in X-ray orientation preferences. While both are stable, is one more accurate under shifts than the other?

A second class of proactive methods are \emph{imperative} in nature. These methods take in datasets collected from multiple, heterogeneous environments and automatically extract invariant predictors from the data without user input \citep{rojas2018invariant,arjovsky2019invariant,bellot2020generalization,koyama2020out}. Examples of these methods are those that compute features sets \citep{rojas2018invariant} or representations \citep{arjovsky2019invariant,bellot2020generalization} that yield invariant predictors. In the X-ray example, imperative methods would require datasets collected from a large number of health centers which diversely represented the sets of shifts that could be observed (e.g., the datasets differ in terms of scanner manufacturer, X-ray orientation, and encoding schemes).  An advantage of such approaches is that they do not require domain expertise in order to determine invariances. These methods often provide theoretical guarantees about minimax optimal performance (i.e., that they have the smallest worst-case error) across the input distributions. Thus, using an imperative approach a model developer can guarantee good worst-case performance at new hospitals that ``look like'' a mixture of the training hospitals. However, they generally do not provide guarantees about stability to a set of specified shifts: we do not know the ways in which the datasets differ (or by how much), so we cannot answer if the model is stable to shifts in scanner manufacturers, X-ray orientations, or encoding schemes.

The difficulties of understanding model behavior within and across each thread of work prevent rigorous analysis of the reliability of models under dataset shifts. In reality, we are presented with a prediction problem in which the data has been collected and generated under some DGP. Dataset shifts can then lead to changes to arbitrary pieces of the DGP. Thus, there is a need for a framework that enables us to answer the fundamental questions about model behavior under changes to the DGP. This would provide common ground to compare algorithms which address dataset shift, and to generate generalizable insights from methods which address particular instances of shifts.

\subsection{Contributions}
In this paper, we provide a unifying framework for specifying dataset shifts that can occur, analyzing model stability to these shifts, and determining conditions for achieving the lowest worst-case error (i.e., minimax optimal performance) across environments produced by these shifts. This provides common ground so that we can begin to answer fundamental questions such as: To what dataset shifts are the model's predictions stable vs unstable? Has the model learned a predictive relationship that is stable to a set of prespecified shifts of interest? How will the model's performance be affected by these shifts?  For models trained using different methods, what guarantees do they provide about stability and accuracy? 

The framework centers around two key requirements: First, a known causal graphical representation of the environment data generating process which includes specifications of what can shift. These specifications take the form of marked unstable edges in the graph which represent causal dependencies between variables that can shift across environments. We consider arbitrary shifts (i.e., interventions) to causal mechanisms in the graph, as opposed to, e.g.,  constraining shifted distributions to be within bounded norm-balls of the training data. This specification entails commonly studied instances of dataset shift (such as label shifts, covariate shifts, and conditional shifts), but also handles the more general, unnamed shifts that we expect to see in practice. Second, we restrict our analysis to methods whose target distribution can be expressed graphically. This entails algorithms which do not learn intermediate feature representations, but instead operate directly on the observed variables. For models which do not induce a graphical representation, we discuss how our results might be used to probe these models for their stability properties and discuss opportunities for future work to bridge this gap.

Our main contribution is the development of a causal hierarchy of stable distributions, in which distributions at higher levels of the hierarchy guarantee lower worst-case error. The levels of the hierarchy provide insight into how stability can be achieved: the levels correspond to three operators on the graphical representation of the environment, which modify the graph to produce \emph{stable distributions}---the learning targets for stable learning algorithms (Definitions \ref{def:lvl1},\ref{def:lvl2},\ref{def:lvl3}). We further show that the operators have different stability properties: higher level operators more precisely remove unstable edges from the graph when producing stable distributions (Corollary \ref{corollary:precision}). Using this graphical characterization of stability, we provide a simple graphical criterion for determining if a distribution is stable to a set of prespecified shifts: a distribution is stable if it modifies the graph to remove the corresponding unstable edges (Theorem \ref{thm:stable}). We then address questions about the accuracy of different stable solutions by showing how the hierarchy provides a causal characterization of the minimax optimal predictor: the predictor which achieves the lowest worst-case error across shifted environments (Proposition \ref{prop:lvl3-opt}). Surprisingly, we find that frequently studied intervention-invariant solutions generally do not achieve this minimax optimal performance. Finally, we demonstrate through a series of semisynthetic experiments that there is a tradeoff between minimax and average performance. Through these contributions, we provide a common theoretical underpinning for understanding model behavior under dataset shifts: when and how stability to shifts can be achieved, in what ways stable distributions can differ, and how to achieve the lowest worst-case error across shifted environments.

\section{Related Work}
While the focus of this work is on statistical and machine learning models, we briefly discuss concepts related to dataset shift that have been studied in other fields. In particular, \emph{external validity}, or the ability of experimental findings to generalize beyond a single study, has long been an important goal in the social and medical sciences \citep{campbell1963experimental}. For example, practitioners (such as clinicians) who want to assess the results of a randomized trial must consider how the results of the trial relate to their target population of interest. This need has led to much discussion and work on assessing the \emph{generalizability} of randomized trials (see, e.g., \cite{rothwell2010commentary,cole2010generalizing,stuart2011use}). More recently, methodological work in causal inference has focused on \emph{transportability} (see \cite{degtiar2021review} for a review). For example, researchers have developed causal graphical methods for determining when and how experimental findings can be transported from one population or setting to a new one. (see, e.g., \citep{pearl2011transportability,pearl2014external,bareinboim2016causal,dahabreh2019generalizing}). Generalizability is also of importance in economics research \citep{camerer2011promise}, with much methodological work focusing on the problem of \emph{sample selection bias} resulting from non-random data collection (see, e.g., \citep{heckman1974shadow,heckman1979sample,winship1992models}). Selection bias leads to systematic differences between the observed data and the general population, and thus is related to problems of external validity, which consider different environments or populations.

Returning to the focus of this work, the problem of differing train and test distributions in predictive modeling is known as \emph{dataset shift} in machine learning \citep{candela2009dataset}. Classical approaches, such as domain adaptation, assume access to unlabeled samples from the target distribution which they use to reweight training data during learning or extract invariant feature representations (e.g., \cite{huang2007correcting,zhang2013domain,ganin2016domain,gong2016domain}). More recently, work on \emph{statistical transportability} has produced sound and complete algorithms for determining how data from multiple domains can be synthesized to compute a predictive conditional distribution in the target environment of interest \citep{pearl2011transportability,correa2019statistical}.
These methods \emph{reactively} adjust to target data. In many practical scenarios, however, it is not possible to get samples from all possible target distributions. Instead, this requires \emph{proactive} methods that make assumptions about the set of possible target environments in order to learn a model from source data that can be applied elsewhere \citep{subbaswamycounterfactual}. Work on proactive methods has primarily focused on either bounded or unbounded shifts. Bounded shifts have been studied through the lens of \emph{distributional robustness}, often assuming shifts within a finite radius divergence ball (e.g., \cite{sinha2017certifying,duchi2016variance}). \cite{heinze2020conditional} consider robustness to bounded magnitude interventions in latent style features. \cite{rothenhausler2018anchor} consider bounded and unbounded mean-shifts in mechanisms (in which the means of certain variables vary by environment). \cite{oberst2021regularizing} build on this to allow for stochastic mean-shifts. In this paper, we focus on unbounded shifts: in safety-critical domains the cost of failure is high and it can be difficult to accurately specify bounds.\footnote{Though the focus of this paper is on arbitrary shifts, we note that a ``stability accuracy tradeoff'' has also been observed in works on bounded distributional robustness. See \citep{duchi2021learning} for an example.}

Many proactive methods use datasets from multiple source environments to train invariant models to predict in new, unseen environments. \cite{muandet2013domain} propose a kernel method to find a data transformation minimizing the difference between feature distributions across environments while preserving the predictive relationship. Invariant risk minimization (IRM) learns a representation such that the optimal classifier is invariant across the input environments \citep{arjovsky2019invariant, ahuja2020invariant}. Similar works learn invariant predictors by seeking derivative invariance across environments \citep{bellot2020generalization,koyama2020out}. These methods learn representations, so the predictors do not induce a graphical representation. Despite this, we discuss how our graphical results can be applied to probe these methods for their stability properties.  IRM establishes its ability to generalize to new environments using  the framework of \emph{invariant causal prediction} (ICP) \citep{peters2016causal}. Under ICP, the minimax optimal performance of a causally invariant predictor is tied to assuming shifts occur in all variables except the target prediction variable. In reality, however, specific shifts occur and we want to determine which ones to protect against (and how). In this work, we take this approach by starting with the data generating process, and derive stable solutions to prespecified sets of shifts.

Other proactive methods make explicit use of connections to causality. Related to ICP, \cite{rojas2018invariant} use multiple source environment datasets to find a feature subset that yields an invariant conditional distribution. This has also been extended to the reactive case in which unlabeled target data is available (see also \cite{magliacane2018domain}). For discrete variable settings in which data from only one source environment are available and there is no unobserved confounding, covariate balancing techniques have been used to determine the causal features that yield a stable conditional distribution \cite{kuang2018stable,kuang2020stable}. Other causal methods assume explicit knowledge of the graph representing the DGP instead of requiring multiple datasets. Explicitly assuming no unobserved confounders, \cite{schulam2017reliable} protect against shifts in continuous-time longitudinal settings by predicting \emph{counterfactual} outcomes. \cite{subbaswamycounterfactual} find a stable feature set that can include counterfactual variables, assuming linear mechanisms. \cite{veitch2021counterfactual} regularize towards counterfactual invariance so that a model learns to predict only using causal associations. Other works consider using counterfactuals generated by human annotation \citep{kaushik2019learning}, data augmentation (see \citep{kaushik2020explaining,ilse2021selecting} for discussion of the relationship between causality and data augmentation),  or active learning \citep{sundin2019active} to improve model robustness. \cite{subbaswamy2019transport} use \emph{selection diagrams} \citep{pearl2011transportability} to identify mechanisms that can shift and find a stable \emph{interventional} distribution to use for prediction. More recently, end-to-end approaches have been developed which relax the need for the graph to be known beforehand, instead learning it from data \citep{zhang2020domain,subbaswamy2020spec}. In this paper we provide a common ground for understanding the different types of stable solutions and for finding stable predictors with the best worst-case performance.

\section{A Hierarchy of Shift-Stable Distributions}\label{sec:hierarchy}
In this section we present a causal hierarchy of stable distributions that are invariant to different types of dataset shifts. First, we will introduce a general graphical representation for specifying dataset shifts that can occur (Section \ref{subsec:shifts}). We use this representation to give a simple graphical criterion for determining if a distribution is stable to a set of prespecified shifts. Then, we present the hierarchy and show that the levels of the hierarchy correspond to three operators on the graphical representation which modify the graph to produce stable distributions (Section \ref{subsec:hierarchy}). This allows different stable distributions to be compared in terms of how they modify the graphical representation. We further show that the hierarchy is nested, and thus, it has implications on the existence of stable distributions (Section \ref{subsec:hierarchy-consequences}). We begin by introducing necessary background on causal graphs (Section \ref{subsec:prelims}). Proofs of results are in Appendix \ref{app:proofs}.

\subsection{Preliminaries}\label{subsec:prelims}
\subsubsection{Notation}  Throughout the paper sets of variables are denoted by bold capital letters while their particular assignments are denoted by bold lowercase letters. We will consider graphs with directed or bidirected edges (e.g., $\leftrightarrow$). Acyclic will be taken to mean that there exists no purely directed cycle. The sets of parents, children, ancestors, and descendants in a graph $\mathcal{G}$ will be denoted by $pa_\mathcal{G}(\cdot)$, $ch_\mathcal{G}(\cdot)$, $an_\mathcal{G}(\cdot)$, and $de_\mathcal{G}(\cdot)$, respectively (subscript $\mathcal{G}$ omitted when obvious from context). For an edge $e$, $He(e)$ and $Ta(e)$ will refer to the head and tail of the edge, respectively.

\subsubsection{Structural Causal Models} We represent the data generating process (DGP) underlying a prediction problem using acyclic directed mixed graphs (ADMGs), $\mathcal{G}$, which consists of a set of vertices $\mathbf{O}$ corresponding to observed variables and sets of directed and bidirected edges such that there are no directed cycles. Directed edges indicate direct causal relations while bidirected edges indicate the presence of an unobserved confounder (common cause) of the two variables. ADMGs are able to represent directed acyclic graph (DAG) models that contain latent variables. However, the latent variables do not need to be known. For example, if an ADMG has an edge $X \leftrightarrow Y$, then this means that there is some, possibly unknown, mechanism by which $X$ and $Y$ are confounded (i.e., there exists some unobserved variable $U$ such that $X \leftarrow U \rightarrow Y$, but the variable $U$ may be unknown). Thus, using ADMGs, a modeler can reason about the effects of unobserved confounding even if the mechanism or the confounder itself is unknown.

The graph $\mathcal{G}$ defines a Structural Causal Model (SCM) \citep{pearl2009causality} in which each variable $V_i\in\mathbf{O}$ is generated as a function of its parents and a variable-specific exogenous noise variable $U_i$: $V_i = f_i(pa(V_i), U_i)$. The prediction problem associated with the graph consists of a target output variable $Y$ and the remaining observed variables as input features.

As an example, consider the DAG in Fig \ref{fig:pneumonia}a. This DAG corresponds to a simple version of the pneumonia example in \cite{zech2018variable}. The goal is to diagnose pneumonia $Y$ from chest x-rays $Z$ and stylistic features (i.e., orientation and coloring) of the image $X$. The latent variable $W$ represents the hospital department the patient visited. The corresponding ADMG is shown in Fig \ref{fig:pneumonia}b. The unobserved confounder, $W$, has been replaced by a bidrected edge.

\begin{figure}[!t]
\centering
\subfloat[]{%
\begin{tikzpicture}
          \def\unit{1}
            \node (a) at (0.75*\unit, -\unit) [label=right:X,point];
            \node (t) at (-0.75*\unit,  -\unit) [label=left:Y,point];
            \node (c) at (0, -1.5*\unit) [label=below:Z,point];
            \node[state,dashed] (w) at (0, -.25*\unit) {W};
            \path[orange] (w) edge (a);
            \path (w) edge (t);
            \path (t) edge  (c);
            \path (a) edge (c);
        \end{tikzpicture}
}
\qquad
\subfloat[]{%
\begin{tikzpicture}

        \def\unit{1}
             \def\unit{1}
            \node (a) at (0.75*\unit, -\unit) [label=right:X,point];
            \node (t) at (-0.75*\unit,  -\unit) [label=left:Y,point];
            \node (c) at (0, -1.5*\unit) [label=below:Z,point];
            \node[fill=black!0] (w) at (0, -.25*\unit) [point];
            \path[dashed, orange] (w) edge (a);
            \path[dashed] (w) edge (t);
            \path (t) edge  (c);
            \path (a) edge (c);
            \end{tikzpicture}
}
\caption{(a) Posited DAG for the pneumonia example of \cite{zech2018variable}. $Y$ represents the target condition, pneumonia. $X$ represents the x-ray style features (e.g., orientation and color encoding scheme). $Z$ represents the x-ray itself. $W$ represents the hospital department the patient visited.  The orange edge represents the unstable style feature mechanism, while the dashed node represents an unobserved variable. (b) The corresponding ADMG for the DAG in (a). The unobserved confounder $W$ has been replaced by a bidirected edge.}
\label{fig:pneumonia}
\end{figure}

\subsection{Stability and Types of Dataset Shifts}\label{subsec:shifts}
In this section we introduce types of dataset shifts that have been previously studied. Then, we graphically characterize instability in terms of edges in the graph of the data generating process. This will be key to the development of the hierarchy in Section \ref{subsec:hierarchy}.

To define the types of dataset shifts, assume that there is a set of environments such that a prediction problem maps to the same graph structure $\mathcal{G}$. However, each environment is a different instantiation of that graph such that certain mechanisms differ. Thus, the factorization of the data distribution is the same in each environment, but terms in the factorization corresponding to shifts will vary across environments. As an example, consider again the graph in Fig \ref{fig:pneumonia}a. In the pneumonia example, each department has its own protocols and equipment, so the style preferences $P(X|W)$ vary across departments. In this example, a \emph{mechanism shift} in the style mechanism $P(X|W)$ leads to differences across environments.

\begin{definition}[Mechanism shift]
A shift in the mechanism generating a variable $V$ corresponds to arbitrary changes in the distribution $P(V|pa(V))$.
\end{definition}

Causal mechanism shifts produce many previously studied instances of dataset shift. Consider, for example, \emph{label shift}, a well-studied mechanism shift in which the distribution of the features $X$ given the label $Y$ ($P(X|Y)$) is stable, but $P(Y)$ varies across environments. Label shift corresponds to a causal graph $Y \rightarrow X$ in which the features are caused by the label, and the mechanism that generates $Y$ varies across environments, resulting in changes in the prevalence $P(Y)$ (see, e.g., \cite{scholkopf2012causal,zhang2013domain}). 

More generally, mechanism shifts are the most common and general type of shift considered in prior work on proactive approaches for addressing dataset shift \citep{peters2016causal,rojas2018invariant,magliacane2018domain,kuang2018stable,subbaswamy2019transport}. However, special cases of mechanism shifts have also been studied. For example, \cite{meinshausen2018causality,rothenhausler2018anchor,oberst2021regularizing} considered parametric \emph{mean-shifted mechanisms}, in which the means of variables in linear SCMs can vary by environment.

\begin{definition}[Mean-shifted mechanisms]
A mean-shift in the mechanism generating a variable $V$ corresponds to an environment-specific change in the intercept of its linear structural equation $V = \text{intercept}_{env} + \sum_{X\in pa(V)} \lambda_{xv} X + u_v$. Nonlinear generalizations are possible.
\end{definition}

Another special case considered by \cite{subbaswamycounterfactual} is \emph{edge-strength shifts}, in which the relationship encoded by a subset of edges into a variable may vary. Variation along an individual edge corresponds to the \emph{natural direct effect} \citep[Chapter 4]{pearl2009causality}. Thus, an edge-strength shift is a mechanism shift which changes the natural direct effect associated with the edge.
\begin{definition}[Edge-strength shift]
An edge-strength shift in edge $X\rightarrow V$ corresponds to a change in the \emph{natural direct effect}: for $ Y = pa(V)\setminus X$ we have that $E[V(x^\prime, Y(x)) - V(x)]$ changes, where $V(x)$ is the counterfactual value of $V$ had $X$ been $x$, and $V(x^\prime, Y(x))$ is the counterfactual value of $V$ had $X$ been $x^\prime$ and had $Y$ been counterfactually generated under $X=x$.
\end{definition}

\textbf{Key Result:} \emph{All of these shifts can be expressed in terms of edges.} First, edge-strength shifts directly correspond to particular edges. Next, since the mechanism generating a variable $V$ is encoded graphically by all of the edges into $V$, shifts in mechanism can be represented by marking all edges into $V$ as unstable. For shifts in mechanism to an exogenous variable $V$ with no parents in the graph, one might imagine adding an explicit mechanism variable $M_V$ to the graph and considering the edge $M_V \rightarrow V$ to be unstable. Finally, mean-shifts correspond to an edge $A \rightarrow V$ where the mean of $V$ is shifted in each environment  $A$ (also referred to as an ``anchor'', see \citep{rothenhausler2018anchor} for a discussion of anchor variables). Thus, mean-shifts are an example of a specific type of edge shift. While the edge representation of shifts is more general, we note that it cannot differentiate between specific instances of shifts (e.g., a mean-shift and a shift in the natural direct effect of an ``anchor'' variable will have the same graphical representation).

We denote the set of \emph{unstable edges} that can vary across environments by $E_u \subseteq E$ where $E$ is the set of edges in $\mathcal{G}$. Graphically, unstable edges will be colored.

\begin{definition}[Unstable Edge]\label{def:unstable-edge}
An edge is said to be unstable if it is the target of an edge-strength shift or a mechanism shift.
\end{definition}

The concept of unstable edges provides a flexible and extensible way to graphically represent dataset shifts.

\subsubsection{Extensions to New Types of Shifts:} We note that defining shifts in terms of unstable edges makes it possible to tackle new problems determined by shifts in sets or paths of unstable edges. For example, DAGs can be used to represent non i.i.d. network data in which certain edges represent \emph{interference} between units (e.g., friendship ties in social networks) \citep{ogburn2014causal,sherman2020intervening}. Thus, one can define dataset shifts pertaining to networks (e.g., deleting, adding, or changing the strength of friendships). Similarly, dataset shifts due to changing path-specific effects \citep{avin2005identifiability} are another interesting avenue for future exploration (e.g., reductions in side effects of a drug while maintaining its efficacy).

While the focus of this paper is on predictive modeling, we note that the shifts we describe have the opportunity to interact with causal inference work on \emph{transportability} and the \emph{``data fusion problem''} \citep{bareinboim2016causal}. There has been much methodological work on causal graphical methods for transporting causal effect estimates (see, e.g., \citep{pearl2011transportability, bareinboim2012transportability,bareinboim2013meta,lee2020general,lee2020generalized}). These works have primarily considered transporting causal effects under mechanism shifts. The proposed edge-based definitions of shifts can help frame transportability problems under new types of edge-based shifts such as those motivated above.

\subsubsection{Stable Distributions}
We can now define \emph{stable distributions}, which are the target sought by methods addressing instability due to shifts. We will refer to a model of a stable distribution as a \emph{stable predictor}.

\begin{definition}[Stable Distribution]
Consider a graph $\mathcal{G}$ with unstable edges $E_u$ defining a set of environments (different data distributions that factorize with respect to $\mathcal{G}$ that have been generated by differences in the mechanisms associated with $E_u$). A distribution $P(Y|\mathbf{Z})$ is said to be stable if for any two environments, $\mathcal{E}_1, \mathcal{E}_2$, that are instantiations of $\mathcal{G}$, $P_{\mathcal{E}_1}(Y|\mathbf{Z}) = P_{\mathcal{E}_2}(Y|\mathbf{Z})$ holds. The distribution $P(Y|\mathbf{Z})$ is not restricted to being an observational distribution.
\end{definition}

Having established a common graphical representation for arbitrary shifts of various types, we provide a graphical definition of stable distributions. First, define an active \emph{unstable path} to be an active path (as determined by the rules of $d$-separation \citep{pearl1988probabilistic}) that contains at least one unstable edge. \textbf{Key Result:} \emph{The non-existence of active unstable paths is a graphical criterion for determining a distribution's stability.}

\begin{restatable}
{theorem}{stable}\label{thm:stable}
$P(Y|\mathbf{Z})$ is stable if there is no active unstable path from $\mathbf{Z}$ to $Y$ in $\mathcal{G}$ and the mechanism generating $Y$ is stable.
\end{restatable}

Intuitively, Theorem \ref{thm:stable} means that a stable distribution cannot capture a statistical association that relies on the information encoded by an unstable edge. In the pneumonia example of Fig \ref{fig:pneumonia}a, the $W\rightarrow X$ edge which denotes the X-ray style mechanism was determined to be unstable. Because $W$ is unobserved, a model of $P(Y|X,Z)$ will learn an association between $Y$ and $X$ through $W$. Thus, $P(Y|X,Z)$ contains an active unstable path, and this distribution is unstable to shifts in the style mechanism. This means that $P(Y|X,Z)$ is different in each environment. By contrast, if $W$ were observed and we could condition on it, then $P(Y|X,Z,W)$ is stable to shifts in the style mechanism because all paths containing the unstable edge are blocked by $W$. Thus, $P(Y|X,Z,W)$ is invariant across environments.

In the next section we use this edge-based graphical characterization to show that all stable distributions, including those found by existing methods, can be categorized into three levels. Thus, this hierarchy defines the ways in which it is possible to achieve stability to shifts.
\begin{figure*}[!t]
\centering
\subfloat[]{%
\begin{tikzpicture}
          \def\unit{1}
          \node (v) at (-1.5*\unit, -\unit) [label=left:V,point];
            \node (a) at (0.6*\unit, -\unit) [label=right:X,point];
            \node (t) at (-0.6*\unit,  -\unit) [label=below:Y,point];
            \node (c) at (0, -1.6*\unit) [label=below:Z,point];
             \path[bidirected] (t) edge[bend left=70] (a);
            \path (v) edge (t);
            \path (t) edge  (c);
            \path (a) edge (c);
            \path[orange] (t) edge (a);
        \end{tikzpicture}
}
\subfloat[]{%
\begin{tikzpicture}
          \def\unit{1}
          \node (v) at (-1.5*\unit, 1*\unit) [label=left:V,point];
            \node (t) at (-0.6*\unit,  1*\unit) [label=right:Y,point];
            \path (v) edge (t);
        \end{tikzpicture}
}
\subfloat[]{%
\begin{tikzpicture}
          \def\unit{1}
          \node (v) at (-1.5*\unit, -\unit) [label=left:V,point];
            \node (a) at (0.6*\unit, -\unit) [label=right:do(X),point];
            \node (t) at (-0.6*\unit,  -\unit) [label=below:Y,point];
            \node (c) at (0, -1.6*\unit) [label=below:Z,point];
             \path (v) edge (t);
            \path (t) edge  (c);
            \path (a) edge (c);
        \end{tikzpicture}
}
\subfloat[]{%
\begin{tikzpicture}
          \def\unit{1}
          \node (v) at (-1.5*\unit, -\unit) [label=left:V,point];
            \node (a) at (0.6*\unit, -1.25*\unit) [label=right:do(X),point];
            \node (t) at (-0.6*\unit,  -\unit) [label=below:Y,point];
            \node (c) at (0, -1.6*\unit) [label=below:Z(X),point];
            \node (p) at (0.6*\unit, -0.5*\unit) [label=right:X(0),point];
             \path[bidirected] (t) edge[bend left=70] (p);
             \path (v) edge (t);
            \path (t) edge  (c);
            \path (a) edge (c);
 \end{tikzpicture}
}
\caption{(a) Example graph for a data generating process in which the orange $Y\rightarrow X$ edge is unstable. (b) The level 1 operator applied to the graph in (a). The stable level 1 distribution over observed variables is $P(Y|V)$, which ignores all information from $X$ and $Z$. (c) Level 2 operator applied to (a), deleting the edges into $X$. The level 2 stable distribution is $P(Y|V,Z,do(X))$. The level 2 operator deletes the stable $Y \leftrightarrow X$ edge. (d) Level 3 operator applied to (a). The stable level 3 distribution is $P(Y|V, Z(X=x), X(Y=0))$, where $Z(X=x)$ is the value of $Z$ had $X$ been set to its observed value $x$, and $X(Y=0)$ is the counterfactual value of $X$ had $Y$ been set to $0$. The counterfactual operator retains the stable $Y \leftrightarrow X$ edge.}
\label{fig:ex1}
\end{figure*}

\subsection{Hierarchy of Shift-Stable Distributions}\label{subsec:hierarchy}
Many works seek stable distributions in order to make predictions that are stable or invariant to dataset shifts. However, because these methods have been developed in isolation, there has been little discussion of whether these methods find the same stable distributions, or how these distributions differ from one another. As a main contribution of this paper, we show in this section that there exists a hierarchy of stable distributions, in which stable distributions at different levels have distinct graphical properties. Thus, the development of this hierarchy provides a common theoretical underpinning for understanding when and how stability to shifts can be achieved, and in what ways stable distributions can differ. In this section we will define the levels of the hierarchy and show that they correspond to different operators that can remove unstable edges from the graph. Then, in the next section we will further study how differences between levels of the hierarchy affect worst-case performance across environments.

\subsubsection{Levels of the Hierarchy}
Armed with the graphical characterization of stability from the previous section,  we now introduce a hierarchy of the 3 categories of stable distributions. The levels of the hierarchy are: 1) observational conditionals, 2) conditional interventionals, and 3) counterfactuals. This hierarchy is related to the hierarchy of causal queries, which defines three levels of causal study questions an investigator can have: association, intervention, and counterfactuals \citep{pearl2009causality}. Also relatedly, in \citep{shpitser2016causal} the authors connect the identification of different types of causal effects to a hierarchy of graphical interventions: node, edge, and path interventions. While these works develop hierarchies that relate different types of causal queries and effects, in this paper we develop a hierarchy of shift-stable distributions which connects different types of stable distributions to interventions which remove unstable parts of the data generating process from the underlying graph of the DGP.

Each level of the hierarchy of stable distributions corresponds to graphical operators which
differ in the precision with which they can remove edges in the graph (Corollary \ref{corollary:precision}, main result of this subsection). Using the graph in Fig \ref{fig:ex1}(a) as a common example, we discuss each level in detail below. Note that in Fig \ref{fig:ex1}(a), the goal is to predict $Y$ from $V,X,Z$, and the $Y\rightarrow X$ edge is unstable.


\begin{definition}[Stable Level 1 Distribution]\label{def:lvl1}
Let $\mathcal{G}$ be an ADMG with unstable edges $E_u$ defining a set of environments $\mathcal{E}$. A stable level 1 distribution is an observational conditional distribution of the form $P(Y|\mathbf{Z})$ such that, for any two environments $\mathcal{E}_1, \mathcal{E}_2 \in \mathcal{E}$, $P_{\mathcal{E}_1}(Y|\mathbf{Z}) = P_{\mathcal{E}_2}(Y|\mathbf{Z})$ holds.
\end{definition}

\textbf{Level 1:} Methods at level 1 of the hierarchy seek invariant conditional distributions of the form $P(Y|\mathbf{Z})$ that use a subset of observed features for prediction \citep{rojas2018invariant,magliacane2018domain}. These distributions only have conditioning (i.e., the standard rules of $d$-separation) as a tool for disabling unstable edges. For this reason, \textit{the conditioning operator is coarse and removes large pieces of the graph}. Consider Figure \ref{fig:ex1}a, in which the maximal stable level 1 distribution is $P(Y|V)$, since conditioning on either $X$ or $Z$ activates the path through the unstable (orange) edge. The conditioning operator disables all paths from $X$ and $Z$ to $Y$ to produce Fig \ref{fig:ex1}b. While the operator successfully removes the unstable edge, many stable edges were removed as well.

\begin{definition}[Stable Level 2 Distribution]\label{def:lvl2}
Let $\mathcal{G}$ be an ADMG with unstable edges $E_u$ defining a set of environments $\mathcal{E}$. A stable level 2 distribution is a conditional interventional distribution of the form $P(Y|do(\mathbf{W}), \mathbf{Z})$ such that, for any two environments $\mathcal{E}_1, \mathcal{E}_2 \in \mathcal{E}$, $P_{\mathcal{E}_1}(Y|do(\mathbf{W}),\mathbf{Z}) = P_{\mathcal{E}_2}(Y|do(\mathbf{W}),\mathbf{Z})$ holds.
\end{definition}

\textbf{Level 2:} Methods at level 2 \citep{subbaswamy2019transport} find conditional interventional distributions \citep{pearl2009causality} of the form $P(Y|do(\mathbf{W}), \mathbf{Z})$. In addition to conditioning, level 2 distributions use \emph{the $do$ operator, which deletes all edges into an intervened variable} \citep{pearl2009causality}. Fig \ref{fig:ex1}c shows the result of $do(X)$ applied to Fig \ref{fig:ex1}a: the edges into $X$ (including the unstable edge) are removed. Thus, $P(Y|Z, V, do(X))$ is stable and retains statistical information along stable paths from $Z$ and $X$ that the level 1 distribution $P(Y|V)$ did not. However, the stable $Y \leftrightarrow X$ edge was also removed by the operator. Intervening interacts with the factorization (according to the graph) of the joint distribution of the observed variables $P(\mathbf{O})$ by deleting the terms corresponding to mechanisms of the intervened variable: $P(Y|Z,V,do(X)) \propto P(Y|V)P(Z|X,Y)$ in Fig \ref{fig:ex1}a. The term $P(Z|X,Y)$ corresponds to the stable information we retain by intervening that we could not capture by conditioning.

\begin{definition}[Stable Level 3 Distribution]\label{def:lvl3}
Let $\mathcal{G}$ be an ADMG with unstable edges $E_u$ defining a set of environments $\mathcal{E}$, and let $Z(w)$ denote the counterfactual value of $Z$ had $W$ been set to $w$ for variables $Z,W\subseteq \mathbf{O}$. A stable level 3 distribution is a counterfactual distribution of the form $P(Y(\mathbf{W}, \mathbf{Z}(\mathbf{W'}))| \mathbf{Z}(\mathbf{W}))$ such that, for any two environments $\mathcal{E}_1, \mathcal{E}_2 \in \mathcal{E}$, $P_{\mathcal{E}_1}(Y(\mathbf{W}, \mathbf{Z}(\mathbf{W'}))| \mathbf{Z}(\mathbf{W}))$ $= P_{\mathcal{E}_2}(Y(\mathbf{W}, \mathbf{Z}(\mathbf{W'}))| \mathbf{Z}(\mathbf{W}))$ holds.
\end{definition}

\textbf{Level 3:} Finally, level 3 methods \citep{subbaswamycounterfactual,veitch2021counterfactual} seek counterfactual distributions, which allow us to consider conflicting values of a variable, or to replace a mechanism with a new one. For example, let $Y$ and $Z$ denote two children of a variable $X$. If we hypothetically set $X$ to $x'$ for $X \rightarrow Y$ but left $X$ as its observed value $x$ for $X \rightarrow Z$, this corresponds to counterfactual $Y(x')$ and factual $Z(x)=Z$. By setting a variable to a reference value (e.g., $0$) for one edge but not others, \emph{computing counterfactuals effectively removes (or replaces) a single edge}. In Fig \ref{fig:ex1}c, we saw $P(Y|V, Z, do(X))$ is stable and deletes both edges into $X$, including the stable $Y \leftrightarrow X$ edge. However, if we compute the counterfactual $X(Y=0)$, depicted in Fig \ref{fig:ex1}d, then the level 3 distribution $P(Y|X(Y=0),V,Z(X=x))$ is stable and only deletes the unstable $Y\rightarrow X$ edge, retaining information along the $Y \leftrightarrow X$ path. More generally, level 3 distributions allow us to counterfactually replace mechanisms (and thus replace the influence along unstable edges) with new ones. We will exploit this fact in Section \ref{sec:minimax} when we investigate accuracy. The effects of the three operators produce the following result:

\begin{corollary}\label{corollary:precision}
Distributions at increasing levels of the hierarchy of stability grant increased precision in disabling individual edges (and thus paths).
\end{corollary}

\textbf{Key Result:} \emph{Thus, the difference between operators associated with the different levels of stable distributions is the precision of their ability to disable edges into a variable.} Level 1, conditioning, must remove large amounts of the graph to disable edges. Level 2, intervening, deletes all edges into a variable. Level 3, computing  counterfactuals, can precisely disable a single edge into a variable. Since paths encode statistical influence this also provides a natural definition for a \emph{maximally stable distribution} as one which deletes the unstable edges, and only the unstable edges. Thus, given a stable distribution found by any method, we can compare to the maximally stable distribution to see which, and how many, stable paths were removed.

Another important fact is that the hierarchy is \emph{nested}. This means that a level 1 distribution can be expressed as a level 2 distribution (and a level 2 distribution can be expressed as a level 3 distribution):

\begin{restatable}[\cite{subbaswamy2019transport}, Corollary 1]
{lemma}{leveltwo}\label{lemma:level-two}
A stable level 1 distribution of the form $P(Y|\mathbf{Z})$ can be expressed as a stable level 2 distribution of the form $P(Y|\mathbf{Z}', do(\mathbf{W}))$ for $\mathbf{Z}'\subseteq\mathbf{Z}\subseteq\mathbf{O}$, $\mathbf{W}\subseteq \mathbf{O}$.
\end{restatable}

\begin{restatable}
{lemma}{levelthree}\label{lemma:level-three}
A stable level 2 distribution of the form  $P(Y|\mathbf{Z}', do(\mathbf{W}))$ can be expressed as a stable level 3 distribution of the form $P(Y(\mathbf{W})|\mathbf{Z}'(\mathbf{W}))$.
\end{restatable}

\subsubsection{Consequences}\label{subsec:hierarchy-consequences}
There are a number of practical consequences of the hierarchy of shift-stable distributions: First, level 1 distributions can always be learned from the available data because conditional distributions are \emph{observational} quantities. This means that we can simply fit and learn a model of $P(Y|\mathbf{Z})$ from the training data. However, because the conditioning operator throws away large parts of the graph, including many stable paths, models of level 1 distributions will generally have higher error compared to models of level 2 and level 3 distributions. A tradeoff exists, though, since level 2 and level 3 distributions are not always \emph{identifiable}---they cannot always be estimated as a function of the observational training data. Level 2 distributions model the effects of hypothetical interventions, and, just as in causal inference, unobserved confounding can lead to identifiability issues (for more detail on identification and level 2 stable distributions see \cite{subbaswamy2019transport}). In addition to identifiability challenges, level 3 counterfactual distributions require further assumptions about the functional form of the causal mechanisms in the SCM. Under a fully specified SCM (i.e., the functions defining mechanisms and their parameters are all known), counterfactual inference can be performed using a three step abduction, action, prediction procedure described in \cite[Chapter 7]{pearl2009causality}. For example, the method of \cite{subbaswamycounterfactual} assumes linear causal mechanisms to compute level 3 distributions. However, we often have limited information about functional forms and the distribution of the exogenous noise variables in an SCM. If we want to make counterfactual queries with fewer or no parametric assumptions, then identifiability becomes even more difficult: In general, not all counterfactual queries will be testable. That is, \textit{experimental} data cannot be used to uniquely verify the result of a counterfactual query (where as experimental data can verify the result of any interventional query). For non-parametric SCMs, \cite{shpitser2007counterfactuals} provide an algorithm for determining if a counterfactual query is empirically testable. Thus, one must balance strong parametric assumptions about the form of causal mechanisms against the possibility of untestable counterfactuals.

The nested nature of the hierarchy means that it has consequences on the existence of stable distributions: If there is no stable level 3 distribution, then no stable level 1 or level 2 distributions exist. Considering the other direction, if we find that no stable level 1 distribution exists, there may still be a stable level 2 or 3 distribution. This is an important consideration as more methods for finding stable distributions are developed. For example, \cite{subbaswamy2019transport} developed a sound and complete algorithm for finding stable level 2 distributions in a graph. This means that the algorithm returns a distribution \emph{if and only if} a stable level 2 distribution exists. \emph{An open problem is to develop a sound and complete algorithm for finding stable level 3 distributions.} Such a result would be very powerful: If a complete algorithm failed to find a stable level 3 distribution, then that would mean no stable distributions (level 1, 2, or 3) exist.

We have shown that the hierarchy of stable distributions defines graphical operators which can be used to construct stable distributions by disabling edges in the the underlying graph. Next, we show how the ability of counterfactual level 3 distributions to replace edges can be used to achieve minimax optimal performance under dataset shifts.

\section{Worst-case Performance of Shift-Stable Distributions}\label{sec:minimax}
We now compare stable distributions with respect to their minimax performance under dataset shifts. Specifically, we show that there is a hypothetical environment in which counterfactually training a model would yield minimax optimal performance across environments. We further show that this level 3 counterfactual distribution is not, in general, a level 2 interventional distribution. Counter to the increasing interest in invariant interventional solutions like Invariant Risk Minimization and its related follow-ups (e.g., \cite{arjovsky2019invariant,bellot2020generalization,koyama2020out}), these results motivate the development of counterfactual (as opposed to level 2) learning algorithms.

\subsection{A Decision Theoretic View}
We now present our result characterizing the stable distribution that achieves minimax optimal performance. First, recall that dataset shifts result in a set of hypothetical environments $\mathcal{E}$ generated from the same graph $\mathcal{G}$ such that the mechanisms associated with unstable edges in $\mathcal{G}$ differ in each environment. For simplicity, we will assume that the mechanism of a single variable $W \in\mathbf{X}$ is subject to shifts while the mechanisms of all other variables $\mathbf{V} = \{\mathbf{X}, Y\} \setminus \{W\}$ remain stable across environments. Each distribution in the set of data distributions $\mathcal{U}$ corresponding to each environment factorizes according to $\mathcal{G}$, but differs only in the term $P(W|pa_\mathcal{G}(W))$ which corresponds to the mechanism for generating $W$.

Now consider the following game: Suppose the data modeler (DM) wishes to pick the distribution $B \in \mathcal{U}$ such that the corresponding Bayes predictor $h^*_B$ (i.e., the true $P_B(Y|\mathbf{X})$) minimizes the worst-case expected loss (i.e., worst-case risk) across all distributions in $\mathcal{U}$. This can be written as

\begin{equation}\label{eq:game}
    \inf_{B\in\mathcal{U}} \sup_{Q\in\mathcal{U}} E_Q[\ell(h^*_B, \mathbf{O})].
\end{equation}

Following a game theoretic result \citep[Theorem 6.1]{grunwald2004game}, this game has a solution for bounded loss functions $\ell$ (e.g., the Brier score but not the log loss):
\begin{restatable}{theorem}{game}\label{theorem:game}
Consider a classification problem and suppose $\ell$ is a bounded loss function. Then Equation \ref{eq:game} has a solution, and the maximum generalized entropy distribution $Q^* \in \mathcal{U}$ satisfies $B^*=$ $\arg\inf_{B\in\mathcal{U}} \sup_{Q\in\mathcal{U}} E_Q[\ell(h^*_B, \mathbf{O})]$ $= \arg \sup_{Q\in\mathcal{U}} \inf_{B\in\mathcal{U}}  E_Q[\ell(h^*_B, \mathbf{O})]$  $= Q^*$.
\end{restatable}

\textbf{Key Result: } \emph{That this game has a solution means that $B^*$ is the ``optimal training environment'' such that counterfactually training a predictor in $B^*$ to learn the true $P_{B^*}(Y|\mathbf{X})$ would produce the minimax optimal predictor.} Importantly, this optimal environment $B^*$ depends on the choice of loss function. There are two consequences of this result: First, $P_{B^*}(Y|\mathbf{X})$ is not, in general, a level 2 distribution (and thus level 2 distributions are not, in general, minimax optimal). Second, there is a level 3 distribution which corresponds to $P_{B^*}(Y|\mathbf{X})$ and thus is minimax optimal.

\begin{restatable}{proposition}{lvltwouniform}\label{prop:lvl2-uniform}
The level 2 stable distribution $P(Y|do(W), \mathbf{X}) = P_Q(Y|\mathbf{X})$, where $Q$ is the member of $\mathcal{U}$ such that $W$ has a uniform distribution, i.e., $P(W, pa(W)) = c P(pa(W))$ for $c\in\mathbb{R}^+$.
\end{restatable}
In the appendix we provide a counterexample in which $Q\not = B^*$. This shows that the level 2 stable distribution $P(Y|do(W), \mathbf{X})$ is not minimax optimal.

\begin{restatable}{proposition}{lvlthreeopt}\label{prop:lvl3-opt}
The level 3 distribution $P(Y(W_{B^*})|\mathbf{X}(W_{B^*}))$ equals $P_{B^*}(Y|\mathbf{X})$ and is minimax optimal, where $W_{B^*}$ is the counterfactual $W$ generated under the mechanism associated with the environment $B^*$.
\end{restatable}
Thus, given training data from $P_0 \in\mathcal{U}$, if we could counterfactually learn $P(Y|\mathbf{X})$ in the environment associated with $B^*$ then the resulting predictor would be minimax optimal. This means the stable level 3 distribution $P(Y(W_{B^*})|\mathbf{X}(W_{B^*}))$ produces the best, worst-case performance across environments out of all distributions that could be used for prediction.

\subsection{A Simple Learning Algorithm}

\begin{algorithm}[!th]
 \SetKwInOut{Input}{input}\SetKwInOut{Output}{output}
 \Input{\# of steps $T$, Step size $\eta$, Data $\mathbf{O}$}
 \Output{Robust model parameters $\hat{\theta}$}
  Initialize $\phi^{(1)},\theta^{(1)}$\;
 \For{$t\in{2\dots T}$}{
 $\phi^{(t)} = \phi^{(t-1)} + \eta \nabla_\phi g(\mathbf{O}, \theta^{(t-1)},\phi^{(t-1)}) $\;
 $\theta^{(t)} = \theta^{(t-1)} - \eta \nabla_\theta g(\mathbf{O}, \theta^{(t-1)},\phi^{(t-1)})$\;
 }
 \KwRet $\frac{1}{T}\sum_{t=1}^T \theta^{(t)}$
  \caption{Gradient Descent Ascent}
  \label{alg:gda}
\end{algorithm}

We now consider a simple distributionally robust likelihood reweighting algorithm for learning the minimax optimal level 3 predictor. This approach can serve as a starting point for developing new stable learning algorithms which achieve minimax optimal performance under dataset shift.\footnote{Alternatively, one could try to directly compute the maximum generalized entropy distribution. See, e.g., \cite{van2019robust} for a simple example.}

For simplicity, suppose there are no unobserved confounders (i.e., $\mathcal{G}$ has no bidirected edges). We relax this condition in the Appendix. Then, learning in the environment $B^*$ using training data from $P_0$ can be done by reweighting the training data:
\begin{align*}
E_{B^*}[\ell(f(\mathbf{x}), y)] &=E_{P_0}\left[\frac{P_{B^*}(\mathbf{O})}{P_{P_0}(\mathbf{O})}\ell(f(\mathbf{x}), y)\right] \\
&= E_{P_0}\left[\frac{P_{B^*}(W|pa(W))}{P_{P_0}(W|pa(W))}\ell(f(\mathbf{x}), y)\right],
\end{align*}
assuming full shared support (i.e., overlap between $P_{B^*}(W|pa(W)))$ and $P_{P_0}(W|pa(W))$ for all values of $W, pa(W)$). 

Because the minimax optimal training environment $B^*$ is unknown, we now seek to train the minimax optimal predictor by parameterizing environments and iteratively finding the worst-case environment. Let $h(W, pa(W);\phi) = \frac{P_{Q}(W|pa(W))}{P_{P_0}(W|pa(W))}$ s.t. $h\in [0, \infty)$ and $E[h|pa(W)]=1$ be a reweighting function parameterized by $\phi$. Note that different values of $\phi$ correspond to different hypothetical training environments $Q$. The learning problem becomes
\begin{align}
    &\min_\theta \max_\phi & & g(\mathbf{O}, \theta,\phi)\\
    &s.t. & & g(\mathbf{O}, \theta,\phi) = E_{P_0}[h(W, pa(W); \phi) \ell(f(\mathbf{x};\theta), y]
\end{align}
with model parameters $\theta$. This objective resembles those of distributionally robust methods (e.g., \cite{duchi2016variance}) without restrictions on the density ratio $h$ or the divergence between $P_Q$ and $P_{P_0}$.

While many possibilities exist, perhaps the simplest version of $h$ is to explicitly learn a parametric density model (e.g., logistic regression for discrete $W$) $\hat{P}$ for $P_{P_0}(W|pa(W))$ and use the same density model class to model $P_{Q}(W|pa(W)) = \hat{Q}(W|pa(W); \phi)$.
Algorithm \ref{alg:gda} describes a gradient descent ascent learning procedure (GDA) for this case, which alternates between finding environmental parameters $\phi^{(t)}$ that maximize the risk of the previous prediction model (with parameters $\theta^{(t-1)}$), and finding model parameters $\theta^{(t)}$ that minimize risk under the previously found worst-case environment (with parameters $\phi^{(t-1)}$). It is important note that this general minimax learning problem is often very challenging with complicated convergence and equilibrium dynamics (see, e.g., \cite{daskalakis2017training,daskalakis2018limit,lin2020gradient}). Thus, Algorithm \ref{alg:gda} only serves as a starting point for designing counterfactual level 3 learning algorithms.

\section{Experiments}
\begin{figure}[!t]
\centering
\begin{tikzpicture}[scale=1, transform shape]
          \def\unit{1.25}
            
            \node (d) at (0,0.25) [label=above:D,point];
            \node (s) at (-.75*\unit, -.5*\unit) [label=above:S,point]; 
            \node (o) at (0.75*\unit, -.5*\unit) [label=above:O,point];
            \node (v)at (-0.5*\unit, -1*\unit) [label=below:V,point];
            \node (l) at (0.5*\unit, -1*\unit) [label=below:L,point];
            
            \path (d) edge (s);
            \path (d) edge (o);
            \path[orange] (s) edge (o);
            \path (s) edge (v);
            \path (s) edge (l);
            \path (o) edge (l);
            \path (d) edge (v);
            \path (d) edge (l);
\end{tikzpicture}
\caption{DAG for the sepsis prediction task. The orange edge denotes the unstable edge: lab test ordering policies vary across hospitals.}
\label{fig:sepsis-dag}
\end{figure}

We turn to semisynthetic experiments on a real medical prediction task to demonstrate practical performance implications of the hierarchy. To carefully study model behavior under dataset shifts, we posit a graph of the DGP for this dataset and reweight the data to simulate a large number of dataset shifts. We investigate how the performance of models of stable distributions at different levels of the hierarchy behave as test environments differ from the training environment. Our results show that though level 3 models can produce the best worst-case performance (i.e., minimax optimal), level 2 models may perform better on average. This further highlights that model developers need to carefully choose how they achieve stability.

\subsection{Motivation and Data}
One prominent application of machine learning is patient risk stratification in healthcare. It has been widely noted that developing reliable clinical decision support models is difficult due to changes in clinical practice patterns \citep{agniel2018biases}. The resulting behavior-related associations are often brittle---policies change over time and differ across hospitals---and can cause models to make dangerous predictions if left unaccounted \citep{schulam2017reliable}. We investigate practical implications of the hierarchy on this important risk prediction challenge.

Below we describe our setup which loosely follows the setup of \cite{giannini2019machine} for predicting patient risk of sepsis, a life-threatening response to infection.
We use electronic health record data collected over four years at our institution's hospital. The dataset consists of 278,947 patient encounters that began in the emergency department. The prevalence of sepsis (S) is $2.3\%$. Three categories of variables were extracted: vital signs (V) (heart rate, respiratory rate, and temperature), lab tests (L) (lactate), and demographics (D) (age and gender). For encounters that resulted in sepsis, physiologic data available prior to sepsis onset time was used. For non-sepsis encounters all data available until the time the patient was discharged from the hospital was used. Min, max, and median features were derived for each time-series variable. Unlike vitals, lab measurements are not always ordered (O), so a binary missingness indicator was given. The graph of the DGP is shown in Fig \ref{fig:sepsis-dag}.

\subsubsection{Shifts in Lab Test Ordering Patterns}
Different lab test ordering policies correspond to shifts in the conditional $P(O|s,d)$.
As a result, missingness patterns vary across datasets derived from different hospitals, because the lab test rate can vary from one institution to another \citep{rhee2017incidence}. To compare across datasets corresponding to differing lab testing patterns, we simulated one hundred datasets as follows: For a given test split, we fit a (logistic regression) model of the ordering policy $P(O|s, d)$ (i.e., the $P$ model). Then, for a new ordering policy $Q(O|s, d)$, we reweight the test samples by $\frac{Q}{P}$ to mimic data from a new hospital which differs only in the ordering policy. Reweighting the examples makes it such that the overall distribution of the reweighted test set is different from the distribution of the original test set without perturbing the feature values of individual examples. Thus, the reweighted datasets consist entirely of examples that were observed in the original dataset.

To simulate an edge shift, we created new ordering policies $Q$ by perturbing the coefficient of sepsis in the $P$ model. This corresponds to changing the log odds ratio for sepsis of a patient receiving a lab test. A log odds $< 0$ mean that lab test orders are more likely for non-sepsis patients than for sepsis patients, while a log odds $> 0$ means that lab test orders are more likely for sepsis patients than for non-sepsis patients. To simulate a mechanism shift, we perturbed all coefficients in the $P$ model.

\subsection{Experimental Setup}
Train/test splits were generated via 5-fold cross-validation. Full experimental details are in the Appendix. Models were fit using the Brier score (which for binary classification is $(y - \hat{y})^2$, where $y$ is the true label and $\hat{y}$ is the predicted probability of class $y=1$) as the loss since it is a bounded loss function (required by Theorem \ref{theorem:game}).

\subsubsection{Models}
We consider the four possible models: stable models for each level of the hierarchy and an \emph{unstable} baseline that does not adjust for shifts. In fitting models, any model structure (e.g., random forests, neural networks, etc.) can be used to fit the marginal/conditional distributions. The choice of model does not impact the study conclusions drawn here. For simplicity, we used logistic regression for all models. The level 1 model excludes the lab-derived and lab order features. With respect to the graph in Fig \ref{fig:sepsis-dag}, this effectively deletes the $O$ and $L$ nodes (and all edges into these nodes). The level 2 model is of $P(S|d, v, l, do(o))$, and is an implementation of the ``graph surgery estimator'' \citep{subbaswamy2019transport}. In Fig \ref{fig:sepsis-dag}, the $do$ operator deletes both edges into the $O$ node.  Finally, the level 3 model was trained using Algorithm \ref{alg:gda}, with a logistic regression counterfactual reweighting model. In Fig \ref{fig:sepsis-dag}, this deletes and then replaces the mechanism for $O$ with a new ordering policy mechanism. All models were implemented in \texttt{JAX} \citep{jax2018github}. Full details of how the level 2 and 3 models were fit are in the Appendix.

\subsection{Results}

\begin{figure*}[!t]
\centering
\subfloat[]{
\includegraphics[width=0.45\textwidth]{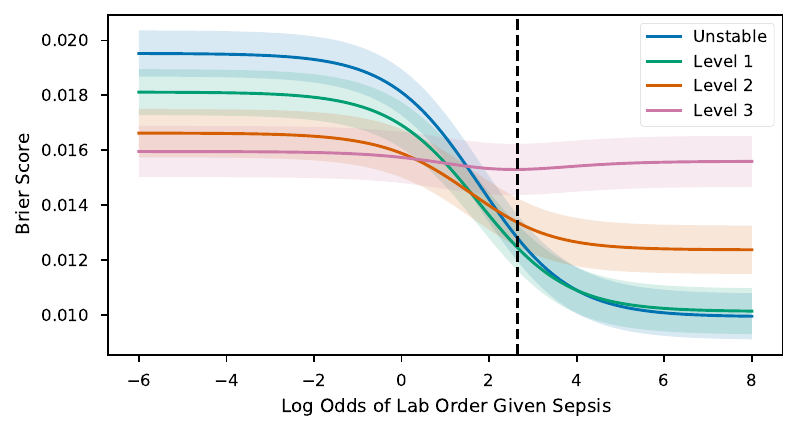}
}\qquad
\subfloat[]{
\includegraphics[width=0.45\textwidth]{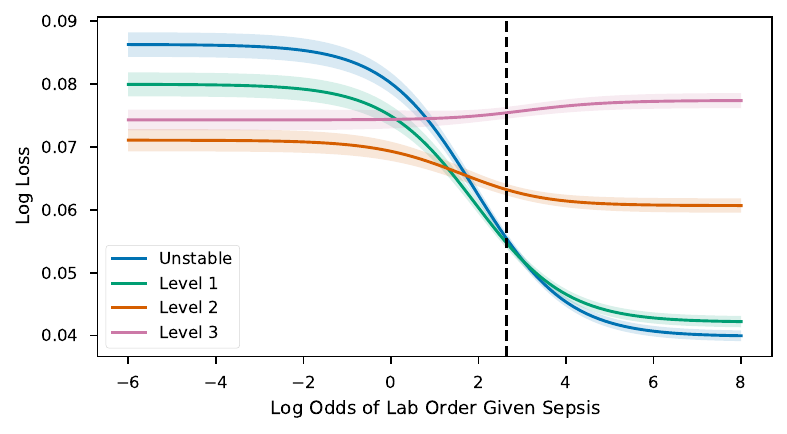}
}

\caption{(a) Performance (Brier score) of different models vs log-odds of lab order for the shift in Fig \ref{fig:sepsis-dag}. Vertical dashed line denotes training value. Shaded regions denote 95\% confidence intervals. (b) Performance (log loss) of different models vs log-odds of lab order for the shift in Fig \ref{fig:sepsis-dag}. Note that the Level 3 model is minimax optimal in (a) while it is not in (b).}
\label{fig:hosp}
\end{figure*}

When test environments differ from the training environment, stable models have more robust performance than unconstrained, unstable models. An unconstrained model uses all dependencies present in the training data; in other words, the model captures correlations due to all paths in the underlying graph. As we impose invariance constraints (by disabling edges), stable models show performance improvements over the unstable model as the test distribution deviates further from the training distribution. We see this, for example, in Fig \ref{fig:hosp}a when, due to the edge shift, the correlation flips from being negative to positive: the level 1, 2, and 3 models outperform the unstable model for log odds ratios $< 1$.

\begin{figure}[!t]
\centering
\includegraphics[scale=0.7]{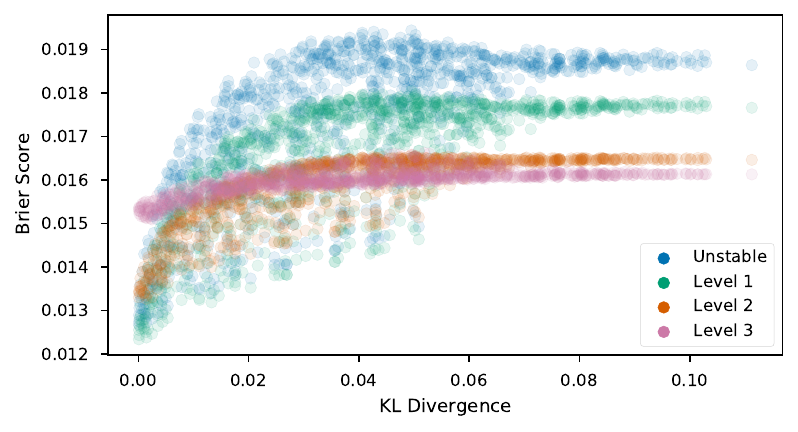}
\caption{Performance (Brier score) of different models vs KL-divergence of new environment distribution from training distribution under a mechanism shift to the lab ordering policy. The Level 3 model is minimax optimal.}
\label{fig:kld}
\end{figure}

As desired, the level 3 model achieves the best worst-case performance amongst the four models, indicating that training using Algorithm \ref{alg:gda} was successful. Further, the performance of the level 3 model is nearly constant across the shifts. This is encouraging evidence, because constant risk is a sufficient condition for a Bayes estimator to be minimax optimal \citep{berger2013statistical}. The results are largely consistent in Fig \ref{fig:kld}, in which we consider a mechanism shift in the lab test ordering policy. We see that irrespective of the KL divergence between the training and shifted distributions, the level 3 model still has almost constant performance.

Finally, from Theorem \ref{theorem:game} we know that the optimal training distribution depends on the choice of loss function (Theorem \ref{theorem:game}). Thus, we do not expect a minimax optimal predictor under one loss to be optimal when measured under a different loss. Indeed, in Fig \ref{fig:hosp}b when the four models are evaluated with respect to the log loss, the level 3 model is no longer minimax optimal. In fact, its performance is strictly worse than that of the level 2 model. Even when evaluated using the Brier score (\ref{fig:hosp}a), the worst-case performance of the level 3 model is only slightly better than the worst-case performance of the level 2 model. Further, the level 2 model sees performance improvements when the log odds increase that the level 3 model does not (loss drops noticeably for x-axis values $> 0$). Thus, on average, the level 2 model might be preferable on this data, and a conservative objective like worst-case performance may not be desirable. This illustrates a classic problem in statistical decision theory: while minimax objectives can be too conservative, it may be difficult to characterize the ``average'' environment or to specify a reasonable prior over environments.

\section{Limitations}
The primary limitation of the framework presented in this paper is its reliance on a known causal graph of the data generating process. Correct specification of the graph is important because the addition of an edge, or the change of orientation of an edge, can change the stability of a distribution. Adding an edge can open new active unstable paths, while a change in orientation of an edge can cause an inactive path to become active (e.g., conditioning on a chain $X \rightarrow Y \rightarrow Z$ we have $X \ci Z | Y$ vs conditioning on a collider $X \rightarrow Y \leftarrow Z$ we have $X \not\ci Z | Y$). As the number of variables increases, it becomes difficult to manually specify an entire causal graph with confidence. In this section, we discuss options for addressing the limitation of misspecification of (or inability to specify) the graph.

When domain knowledge is insufficient to specify a causal graph, one can try to learn the structure of the graph from data, a problem known as \emph{causal discovery} or \emph{structure learning} (see \citep{spirtes2000causation} for an overview). \emph{Constraint-based} structure learning algorithms work by using (conditional) independence tests to determine edge adjacencies. Thus, by testing compatibility with the data, it is possible to learn the structure of the graph up to an \emph{equivalence class}: a set of fully specified graphs which imply the same independences. Notably, some constraint-based structure learning algorithms tolerate and account for the possibility of unknown confounding variables (see \citep{glymour2019review} for a recent review of structure learning algorithms).

Using structure learning, it is possible to learn the range of causal structures which are compatible with the data. Given this range of causal structures, there are two main approaches one could use to find stable distributions. One approach is to enumerate each member of the equivalence class, find and fit models of stable distributions in the fully specified member, and compare across the members of the equivalence class. This approach is akin to sensitivity analysis approaches for finding the range of causal effect estimates in the equivalence class (see, e.g., \citep{maathuis2009estimating,malinsky2017estimating}). The challenge with this approach is that it does not produce a single model that is guaranteed to be stable, but rather a range of candidate ``possibly stable'' models. One would require data from a new environment to test the stability of the candidate models.

A second approach is to find a distribution which is stable in every member of the equivalence class. Such a distribution is guaranteed to be stable, regardless of which member of the equivalence class represents the ``true'' data generating process. While this could be done through enumeration of each member of the equivalence class (as in the previously outlined sensitivity analysis approach), recent approaches allow us to find stable distributions in graphical representations of the equivalence class \citep{zhang2020domain,subbaswamy2020spec}). The output of many constraint-based structure learning algorithms is a \emph{partial} graph in which edges may be partially directed (i.e., edge endpoints may be an arrowhead, an arrow tail, or $\circ$ representing that either is possible). One can then consider extensions of graphical operators from the hierarchy to partial graphs. As one example, \citep{subbaswamy2020spec} propose a method for finding stable level 1 and level 2 distributions in partial graphs. More generally, a promising direction for future work is to extend results from the proposed framework to partial graphs. Because partial graphs can be learned from data, this would relax the requirement of a fully specified graph as the starting point for this graphical framework.

\section{Contrast with Invariant Risk Minimization}
The discussion in this paper has focused on a graphical perspective---explicitly starting with knowledge of the data generating process and using this to determine when and how stability to shifts is achievable. An alternative emerging paradigm in machine learning has focused on \emph{invariant risk minimization} (IRM) \citep{arjovsky2019invariant,koyama2020out,bellot2020generalization}. IRM is applicable when multiple datasets from different environments are available, and the goal is to learn a representation that produces an optimal predictor which is invariant across these environments. In this section, we discuss an important limitation of the invariant risk minimization paradigm which highlights a key advantage of graphical approaches. We also discuss how graphical analyses can guide future work to address this.

A critical question that determines the usefulness of an invariant predictor is: to what set of shifts is the predictor stable? The answer to this question defines the set of new environments to which an invariant predictor can be safely applied. In the graphical approach, the answer is transparent by design. Shifts are defined as (arbitrary) changes to particular causal mechanisms in the graph, so an invariant predictor is exactly one which is stable to the specified shifts in mechanisms. Further, the graph allows model developers to choose the set of shifts to which a predictor should be stable and provide guarantees about shifts that are protected against.

In contrast, IRM methods currently struggle to answer this critical question. First, existing IRM methods do not identify the differences that exist across the observed environments. Thus, they are unable to provide guarantees about the nature of the shifts in environment (i.e., the causal mechanisms) against which they protect. This also means it is difficult to state the set of new environments to which the invariant predictor can be safely applied. Further, because IRM automatically determines invariance from datasets, there is no opportunity for developers to specify particular invariances that they want to hold.

Outside of invariant risk minimization, there are opportunities to leverage ideas from other works on invariant learning and ideas from the proposed graphical framework to improve IRM-type methods. For example, \cite{wald2021calibration} shows a relationship between invariant predictors and calibration across environments. This suggests a possible approach for probing an invariant predictor for stability to particular mechanisms shifts. First, using \emph{structure learning} \citep{spirtes2000causation}, it is possible to detect particular mechanism shifts that occur across environments \citep{zhang2017causal,zhang2020domain,subbaswamy2020spec}. Then, when mechanisms of interest have been identified, one can test for stability to particular mechanisms shifts by examining how the calibration of the predictor changes as evaluation data is reweighted according to the distribution associated with the mechanism shift. This would provide a post-hoc way to verify the integrity of a trained invariant predictor.

As another example, \citep{veitch2021counterfactual} show that counterfactual invariances leave observable distributional signatures that can be used to design regularizers to enforce the given invariance. This motivates the combination of IRM-type objectives with regularizers which explicitly capture desired invariances at different levels of the hierarchy of shift-stability. This would allow developers to specify particular invariances they want to guarantee while also automatically learning other invariances from the data.
In the context of image classification, \citep{heinze2020conditional} show how multiple views of an image and data augmentation can be used to learn models which are invariant to shifts in known and unknown style features. This provides ideas for learning specified invariances in settings with unstructured data (e.g., images and text).

\section{Conclusion}
The use of machine learning in production represents a shift from applying models to static datasets to applying them in the real world. As a result, aspects of the underyling DGP are almost certain to change. Many methods have been developed to find distributions that are stable to dataset shift, but as a field we have lacked common underlying theory to characterize and relate different stable distributions. To address this, we developed a common framework for expressing the different types of shifts as unstable edges in a graphical representation of the DGP. We further showed that stable distributions belong to a causal hierarchy in which stable distributions at different levels have distinct operators that can remove unstable edges in the graph. This provides a new, but natural, way to characterize and construct stable models by only removing unstable edges. This also motivates a new paradigm for future work developing methods that can modify individual edges. We also showed that popular invariant solutions (level 2; invariant under intervention) do not, in general, achieve minimax optimal performance across environments. Our experiments showed that there is a tradeoff between worst-case average performance. Thus, model developers need to carefully determine when and how they achieve invariance.

\appendix

\section{Medical Risk Prediction Experiment}\label{app:real}
\subsection{Data}
Our experimental setup follows that of \cite{delahanty2019development}. The dataset contains electronic health record data collected over four years at our institution’s hospital. The dataset consists of 278,947 emergency department patient encounters. The prevalence of the target disease, sepsis (S), is $2.1\%$. Features pertaining to vital signs (V) (heart rate, respiratory rate, temperature), lab tests (L) (lactate), and demographics (D) (age, gender) were extracted. For encounters that resulted in sepsis (i.e., positive encounters), physiologic data available up until sepsis onset time was used. For non-sepsis encounters, all data available until discharge was used. For each of the time-series physiologic features (V and L), min, max, and median summary features were derived. Unlike vitals, lab measurements are not always ordered (O) and are subject to missingness (lactate $89\%$ missing). To model lab missingness, missingness indicators (O) for the lab features were added, and lab value-missingness interactions terms were used in place of lab value features.

\subsection{Experimental Details}
Logistic Regression models were fit using a custom \texttt{JAX} \cite{jax2018github} implementation. $L_2$ regularization with regularization coefficient $0.1$ was used (hyperparameter chosen via grid search using the performance of the unstable model on a hold-out 10\% of the initial dataset). These same hyperparameters were used to train the Level 1-3 models. For the predictive models, a b-spline basis feature expansion was used for continuous features (lab values and vital signs). Following the standards in \cite{delahanty2019development} for accounting for missingness, the missingness feature and the missingness-lab value interaction features were added.

The specific shift in lab test ordering patterns considered was a shift in lactate ordering, as these patterns have seen great variation across hospitals and are known to be associated with sepsis \cite{rhee2017incidence}. Lactate missingness has a correlation of -0.36 with sepsis in this dataset (i.e., the presence of the measurement is predictive of the target variable).

Thus, to simulate the edge shift, in each test fold, we first fit a logistic regression model (no b-spline basis expansion, with default scikit-learn hyperparameters) to the test fold's lactate missingness ($O=0$) given $S,D$. That is, a logistic regression model of $P(O=0|s,d)$. Then, to simulate the edge shifted lactate ordering policies, we replaced the coefficient for sepsis in the logistic regression model with 100 values on a grid from -6 to 8. The resulting logistic regression model is of the hypothetical shifted hospital's ordering policy $Q(O=0|s, d)$. Evaluating the loss under each shift was then done by using sample weights computed as $\frac{Q_i}{P_i}$ for each test sample using the two models.

The mechanism shift was simulated in a similar manner to the edge shift. However, instead of only perturbing the coefficient of sepsis in the $P(O|s, d)$ model, all coefficients and the intercept were perturbed. Specifically, for a single test fold, 1000 new coefficients were sampled as follows: Let $w$ denote the weight in the $P$ model. The new coefficients/intercepts were drawn from $Unif(-|w|-0.1, |w|+0.1)$. Because all weights of the logistic regression model changed, we plotted the shifts according to the estimated (using the test set) KL-divergence between the $Q$ and $P$ logistic regression models: $E_{P_{S,D}}[KL(P(O|s,d) || Q(O|s,d)]$.

The level 1 model was fit using a reduced feature set that excluded the lactate features (min, max, median) and lactate missingness indicator. The level 2 model $P(S|d,v,l, do(o))$, an instance of the ``graph surgery estimator'' \cite{subbaswamy2019transport}, was fit by inverse probability weighting (IPW). The term corresponding to $O$ in the factorization of the DAG in Fig \ref{fig:hosp}a is $P(O|s,d)$, so we fit a logistic regression model of this distribution using the training data. Then, the main logistic regression prediction model with the full feature set was trained using sample weights $\frac{1}{P(o_i|s_i,d_i)}$. The resulting model was the level 2 model. The level 3 model is similar to the level 2 model, but instead corresponds to a counterfactual ordering policy $Q(O|s,d)$. The level 3 logistic regression model was trained using the Gradient Ascent Descent procedure in Algorithm \ref{alg:gda}. As noted in the main paper, this procedure has complicated dynamics and we found it was quite sensitive to the choice of step size (or learning rate $\eta$). Through grid search using performance on a hold-out 10\% of the initial dataset the value $\eta=5$ was selected. The model parameters $\theta$ were initialized using the learned level 2 model parameters, and the reweighting $\phi$ parameters were initialized via random draws from $\mathcal{N}(0, 0.1^2)$. The resulting model was the level 3 model.

\section{Proofs}\label{app:proofs}

\stable*
\begin{proof}
We first recall that all of the shifts considered in Section \ref{subsec:shifts} are types of arbitrary shifts in mechanism: mean-shifted mechanism are a special parametric case, and edge-strength shifts correspond to a constrained class of mechanism shifts in which the natural direct effect associated with the mechanism has changed. Thus, if a distribution is stable to arbitrary shifts in mechanisms, then it will also be stable to mean-shifts and edge-shifts. For this reason, in our proof we will prove a distribution is stable by leveraging previous graphical results on stability under shifts in mechanisms (and stability to specific cases follows).

To do so, we will leverage results from \emph{transportability}, which uses a graphical representation called \emph{selection diagrams} (see \cite{pearl2011transportability,subbaswamy2019transport} for details). A selection diagram is a a graph augmented with selection variables $\mathbf{S}$ (which each have at most one child) that point to variables whose mechanisms may vary across environments. Prior results have shown that a distribution $P(Y|\mathbf{Z})$ is stable is if $Y \ci \mathbf{S} | \mathbf{Z}$ in the selection diagram (see \cite[Theorem 2]{pearl2011transportability} and \cite[Definition 3]{subbaswamy2019transport}). Thus, to prove the theorem, we will first translate our unstable edge representation of the graph to a selection diagram. Then, we will show that if $Y$ is not $d$-separated from the selection variables that this implies there is an unstable active path to $Y$.

We first translate our unstable edge representation of the graph to a selection diagram. For an edge $e$ let $He(e)$ denote the variables that $e$ points into. Now for each $e\in E_u$, add a unique selection variable that points to each $V = He(e)$. This indicates that the mechanism that generates V is unstable. We now consider the cases in which there could be an active path from a selection variable to $Y$ (which would make a distribution $P(Y|\mathbf{Z})$ unstable), and show that this corresponds to an active path that contains an unstable edge.

There are two possible ways there can be an active path from a variable $S\in\mathbf{S}$ to $Y$. If there is an active forward path from $S$ to $Y$ (e.g., $S \rightarrow ch(S) \rightarrow \dots Y$) then there is a corresponding active path from $Ta(e)$ to $Y$ that contains the unstable edge $e$: e.g., a path $Ta(e) - e \rightarrow ch(S)\rightarrow \dots Y$. Alternatively, an active forward path indicates that the mechanism that generates $Y$ is unstable.

The other case is if there is an active path beginning with a collider from $S$ to $Y$ (e.g., $S \rightarrow ch(S) \leftarrow \dots Y$). Then there is a corresponding active path from $Ta(e)$ to $Y$ that contains $e$: e.g.,  $Ta(e) - e \rightarrow ch(S) \leftarrow \dots Y$. Thus, in a selection diagram if $P(Y|\mathbf{Z})$ is unstable, then there is an active unstable path to $Y$ in the original unstable edge-denoted graph. Taking the contrapositive of this statement proves the theorem.
\end{proof}

\leveltwo*
\begin{proof}
This is a restatement of Corollary 1 in \cite{subbaswamy2019transport}.
\end{proof}

\levelthree*
\begin{proof}
Consider the (level 2) intervention $do(X)=x$. For a variable $V$ letting $V(x)$ denote the value $V$ would have taken had $X$ been set to $x$ we have that $P(V(x))=P(V|do(x))$. When interventions are consistent (i.e., for $x \not = x'$ there are no conflicting interventions $do(X=x)$ and $do(X=x')$) counterfactuals reduce to the \emph{potential responses} of interventions expressible with the $do$ operator \cite[Definition 7.1.4]{pearl2009causality}.
\end{proof}

For completeness, we restate the following result from \cite{grunwald2004game}. For the present paper, both the action space $\mathcal{A}$ and the set of distributions $\Gamma$ are $\mathcal{U}$ (the DM is picking a training distribution (the action $a$ from $\mathcal{U}$ and nature is picking the test distribution $P$ from $\mathcal{U}$).
\begin{theorem}[Theorem 6.1, \cite{grunwald2004game}]
Let $\Gamma \subseteq \mathcal{P}$ be a convex, weakly closed, and tight set of distributions. Suppose that for each $a \in \mathcal{A}$ the loss function $L(x,a)$ is bounded above and upper semicontinuous in x. Then the restricted game $\mathcal{G}^\Gamma=(\Gamma, \mathcal{A}, L)$ has a value. Moreover, a maximum entropy distribution $P^*$, attaining $$\sup_{P\in\Gamma}\inf_{a\in\mathcal{A}} L(P, a),$$ exists.
\end{theorem}

\game*
\begin{proof}
This result follows directly from \cite[Theorem 6.1]{grunwald2004game}.

The preconditions are trivially satisfied: The set of all distributions over $W$ is convex, closed, and tight. We consider bounded loss functions, which for finite discrete $Y$ (i.e., for classification problems) are continuous. Thus, the game has a solution.

Further, by \cite[Corollary  4.2]{grunwald2004game}, the maximum generalized entropy distribution $Q^*$ is also the distribution minimizing the worst-case expected loss.
\end{proof}

\lvltwouniform*
\begin{proof}
We know that every distribution in $\mathcal{U}$ factorizes according to the graph $\mathcal{G}$, and that they only differ in the term corresponding to the mechanism for $W$, $P(W|pa(W))$. Thus, for any $D\in\mathcal{U}$, $P_D(W, pa(W)) = P_D(W|pa(W)) P_D(pa(W)) =P_D(W|pa(W)) P(pa(W))$, noting that $P(pa(W))$ is the same across all members of $\mathcal{U}$. It suffices to show, then, that $P(\mathbf{O}\setminus\{W\}|do(W)) \propto P_Q(\mathbf{O})$ (within a constant factor), such that $P_Q(W,pa(W)) = c P(pa(W))$.

Recall that, by definition, performing $do(W)$ deletes the $W$ term from the factorization (or equivalently sets $P(W=w|do(w), pa(W)) = P(W=w|do(w)) = 1$), resulting in the so called ``truncated factorization.'' Further, the resulting distribution $P(\mathbf{O}\setminus\{W\}|do(W))$ is a proper distribution (sums to 1) over $\mathbf{O}\setminus\{W\}$. Consider two cases: 1) That $W$ is a discrete variable or 2) $W$ is a continuous variable. With slight abuse of notation, for continuous variables the results will be with respect to the pdf.

\begin{enumerate}
    \item Suppose $W$ is discrete and that across environments it is observed to take $k$ distinct values for $k\in\mathbb{N}, k < \infty$. $P(\mathbf{O}\setminus\{W\}|do(W))$ is not a proper distribution over $\mathbf{O}$ because $\sum_w P(W=w|do(w), pa(W)) = k$. However, this can be made proper by by normalizing it such that $P(W=w|do(w), pa(W)) = \frac{1}{k}$. Thus, $P(\mathbf{O}\setminus\{W\}|do(W))$ is within a constant factor of $P_Q(\mathbf{O})$ where $Q$ is the member of $\mathcal{U}$ such that $P(W|pa(W)) = P(W) = \frac{1}{k}$ (i.e., where $W$ has a discrete uniform distribution). W.r.t. the theorem statement, $c = \frac{1}{k}$.
    
    \item This case follows similarly. Suppose $W$ is continuous and that across environments it is observed to be bounded in the interval $[-M, M], 0 < M < \infty$. Then $P(\mathbf{O}\setminus\{W\}|do(W))$ is not a proper density over $\mathbf{O}$ because $\int_{-M}^M f(W=w|do(w), pa(W)) dw = 2M$, but this can be made proper by normalizing the pdf of $P(W|pa(W))$ to be $\frac{1}{2M}$. Thus, the level 2 density is within a constant factor of $Q$, the member of $\mathcal{U}$ where $W$ has a continuous uniform distribution over the interval $[-M, M]$.  W.r.t. the theorem statement, $c = \frac{1}{2M}$.
\end{enumerate}
\end{proof}

\begin{figure}[!th]
\centering
\begin{tikzpicture}
          \def\unit{.75}
            \node (a) at (0.75*\unit, -\unit) [label=right:W,point];
            \node (t) at (-0.75*\unit,  -\unit) [label=left:Y,point];
            \node (c) at (0, -1.5*\unit) [label=below:Z,point];
            \path[orange] (t) edge (a);
            \path (t) edge  (c);
            \path (a) edge (c);
        \end{tikzpicture}
\caption{Graph for the counterexample.}
\label{fig:counterexample}
\vspace{-0.2in}
\end{figure}
\begin{corollary}
Stable level two distributions are not, in general, minimax optimal.
\end{corollary}
\begin{proof}
The following counterexample is adapted from an example in \cite{van2019robust}.

Consider the DAG $\mathcal{G}$ in Fig \ref{fig:counterexample} in which the goal is to predict $Y$ from $W$ and $Z$, and the mechanism for generating $W$ (i.e., $P(W|Y)$) varies across environments. The distribution factorizes as $P(Z, W, Y) = P(Z|W,Y)P(W|Y)P(Y)$.

Let all variables be binary, and assume that $P(Y=1) = \frac{1}{2}$ and $P(Z=1|W,Y) = \frac{1}{2}$ if $Y=X$ and $P(Z=1|W,Y) = 1$ otherwise. Finally, we will parameterize $P(W|Y)$ as follows: $P(W=1|Y=0) = 1-\alpha_0$ and $P(W=1|Y=1) = \alpha_1$ for $\alpha_0,\alpha_1 \in [0, 1]$. For the Brier score, \cite{van2019robust} computed the maximum generalized entropy parameter values to be $\alpha_0= 2-\sqrt{2}$ and $\alpha_1=2-\sqrt{2}$.

Thus, the minimax optimal $P(W|Y)$ that yields the maximum generalized entropy $P(Y|W,Z)$ is $P(W=1|Y=1) = 2-\sqrt{2} \approx 0.586$ and $P(W=1|Y=0) = \sqrt{2} - 1 \approx 0.414$. This is different than the $P(W|Y)$ that yields the $P(Y|W,Z)$ equivalent to the stable level 2 solution $P(Y|do(W), Z)$, which is $P(W|Y) = P(W) = 0.5$ (by Proposition \ref{prop:lvl2-uniform}). Thus, the level 2 solution $P(Y|do(W), Z)$ is not optimal for this graph using the Brier score (this also holds for the log loss; see the computations in \cite{van2019robust}).
\end{proof}

\lvlthreeopt*
\begin{proof}
Given that our training data was generated from $P_0 \in \mathcal{U}$, we are interested in $P(Y|X)$ if counterfactually $W$ had been generated using the mechanism (i.e., we edited the structural equation in the SCM) $g_{B^*}(pa(W),\epsilon_w)$, the mechanism for $W$ associated with the environment $B^* \in\mathcal{U}$ identified in Theorem \ref{theorem:game}. This mechanism change produces a new distribution associated with $W$, $P_{B^*}(W|pa(W))$.

We can represent this counterfactually by letting $Z(W_{B^*}) = Z(g_{B^*}(pa(W))$ be the potential outcome of $Z$ had $W$ been generated according to $g_{B^*}$ for some variable $Z$ (the rhs notation is sometimes used to express policy interventions, see, e.g., \cite{sherman2020intervening}). Thus, the counterfactual distribution can be expressed as $P(Y(W_{B^*})|W_{B^*}, pa(W), \mathbf{X}\setminus\{W, pa(W)\}(W_{B^*})) = P(Y(W_{B^*})|\mathbf{X}(W_{B^*}))$ (noting that $pa(W)(W_{B^*}) = pa(W)$ because changing the mechanism of $W$ does not affect its parents). Because $P_0$ and $B^*$ differ only with respect to the mechanism generating $W$, the counterfactual distribution associated with this mechanism change yields $P_{B^*}(Y|X)$. Thus, $P(Y(W_{B^*})|\mathbf{X}(W_{B^*}))$ is minimax optimal because $P_{B^*}(Y|X)$ was shown to be minimax optimal in Theorem \ref{theorem:game}.
\end{proof}

\section{Likelihood Reweighting In The Presence of Unobserved Confounders}
In Section \ref{sec:minimax} we developed a likelihood reweighting formulation for a minimax optimal predictor by assuming that the graph has no bidirected edges (no unobserved confounders). We now relax this condition.

First, note that the \emph{c-component} (or \emph{district}) of a variable in an ADMG is the set of nodes reachable via purely bidirected paths (i.e., paths of the form $V_1 \leftrightarrow \dots \leftrightarrow V_2$). An ADMG over variables $\mathbf{O}$ factorizes as:
\begin{equation*}
    Q[\mathbf{O}] = P(\mathbf{O}) = \sum_{\mathbf{U}} \prod_{O_i\in\mathbf{O}} P(O_i|pa(O_i), U_i) P(\mathbf{U})
\end{equation*}
where $\mathbf{U}$ are the exogenous noise variables. Note that an ADMG factorizes as a product of Q-factors over the c-components. That is, if $\mathbf{O}$ is partitioned into c-components $\{S_1, S_2, \dots, S_k\}$, then $P(\mathbf{O}) = \prod_1^k Q[S_k]$ \cite[Lemma 7]{tian2002studies}.\footnote{When the graph has no bidirected edges, each node is its own c-component.} Finally, let $O_1 < O_2 < \dots < O_n$ be a topological order over $\mathbf{O}$. Then each c-factor is identifiable and given by $Q[S_j] = \prod_{i|O_i\in S_j} P(O_i|pa(T_i)\setminus\{O_i\})$, where $T_i$ is the c-component of $\mathcal{G}_{O_1\dots O_i}$ that contains $O_i$.

Now we can see that the term $P(O_i|pa(T_i)\setminus\{O_i\})$ is the ADMG generalization of $P(O_i|pa(O_i))$ in DAGs, and is the term associated with the mechanism for generating $O_i$. Thus, if $P(W|do(pa(W)))$ is identifiable in the ADMG, then we we need to perform likelihood reweighting with respect to $P(W|pa(T_W)\setminus\{W\})$. That is, let $B^*\in\mathcal{U}$ be the minimax optimal training distribution/environment. Then 
\begin{align*}
    E_{B^*}[\ell(f(\mathbf{x}, y))] &= E_{P_0}\left[\frac{P_{B^*}(\mathbf{O})}{P_{P_0}(\mathbf{O})}\ell(f(\mathbf{x}, y))\right]\\
    &= E_{P_0}\left[\frac{P_{B^*}(W|pa(T_W)\setminus\{W\})}{P_{P_0}(W|pa(T_W)\setminus\{W\})}\ell(f(\mathbf{x}, y))\right]
\end{align*}
and we can define a reweighting function as before.

\bibliographystyle{vancouver}
\bibliography{references}

\end{document}